\documentclass[aos,preprint]{imsart}
\usepackage[margin=1.5in]{geometry}
\RequirePackage[OT1]{fontenc}
\RequirePackage{amsthm,amsmath}
\RequirePackage[numbers]{natbib}
\RequirePackage[colorlinks = true,
            linkcolor = black,
            urlcolor  = black,
            citecolor = black,
            anchorcolor = black]{hyperref}
\usepackage{graphicx}
\usepackage{enumerate}
\usepackage{amsfonts}
\usepackage{amssymb}
\usepackage{fancyhdr}
\usepackage{comment}
\usepackage{parskip}
\setattribute{journal}{name}{}

\startlocaldefs
\numberwithin{equation}{section}
\theoremstyle{plain}

\newtheorem{theorem}{Theorem}
\newtheorem*{remark}{Remark}
\newtheorem{lemma}{Lemma}
\newtheorem{corollary}{Corollary}

\newtheorem{proposition}[theorem]{Proposition}
\endlocaldefs

\newcommand{\ts}{\theta^*}
\newcommand{\E}{\mathbb{E}}
\renewcommand{\P}{\mathbb{P}}
\renewcommand{\H}{\mathcal{H}}
\newcommand{\B}{\mathcal{B}}
\newcommand{\D}{\mathcal{D}}
\renewcommand{\th}{\hat{\theta}}
\newcommand{\tn}{\th_0}
\newcommand{\R}{\mathbb{R}}

\DeclareMathOperator*{\argmax}{\arg\!\max}

\begin{document}

\begin{frontmatter}
\title{Statistical Guarantees for Estimating the Centers of a Two-component Gaussian Mixture by EM}
\runtitle{Statistical Guarantees for Gaussian Mixtures by EM}

\begin{aug}

\author{\fnms{Jason M. Klusowski}\ead[label=e1]{jason.klusowski@yale.edu}}
\and
\author{\fnms{W. D. Brinda}\ead[label=e2]{william.brinda@yale.edu}}

\affiliation{Yale University, Department of Statistics \\
24 Hillhouse Avenue \\
New Haven, Connecticut, USA \\
\printead{e1}; \printead*{e2} }

\end{aug}

\begin{abstract} \\ \\
\indent Recently, a general method for analyzing the statistical accuracy of the EM algorithm has been developed and applied to some simple latent variable models [Balakrishnan et al. 2016]. In that method, the basin of attraction for valid initialization is required to be a ball around the truth. Using Stein's Lemma, we extend these results in the case of estimating the centers of a two-component Gaussian mixture in $d$ dimensions. In particular, we significantly expand the basin of attraction to be the intersection of a half space and a ball around the origin. If the signal-to-noise ratio is at least a constant multiple of $ \sqrt{d\log d} $, we show that a random initialization strategy is feasible. 
\end{abstract}


\end{frontmatter}

\section{Introduction}

The expectation-maximization (EM) algorithm has had a long and rich history since the seminal paper of Dempster et al. \cite{Dempster1977}. Indeed, even earlier analogs had been used in incomplete-data problems \cite{Beale1975}. Modern applications are commonly seen in latent variable models or when the data is missing or corrupted. Although the EM algorithm is known to have desirable monotonicity and convergence properties \cite{Wu1983}, such features may fail when the likelihood function is multi-modal.

The purpose of this paper is to extend a result from \cite{Balakrishnan2014}, where guaranteed rates of convergence of the EM iterates are given for various simple models. These results all rely on initializing the algorithm in a ball around the unknown parameter of interest. We consider the case of estimating the centers of a two-component Gaussian mixture and enlarge the basin of attraction to the intersection of a half space and a large ball around the origin. In accordance with other work \cite{Dasgupta2007}, we also show that if the degree of separation of the centers scales with the dimension, the basin of attraction is large enough to ensure that random initialization from an appropriately scaled multivariate normal distribution is practical.

In Section~\ref{sec:em-iterates}, we briefly review the EM algorithm and derive the exact form of the operator for our Gaussian mixture example. Section~\ref{sec:error-bounds} contains our main results. We devise a suitable region for which the population EM operator is stable and contractive toward the true parameter value. We then find bounds on the error of the sample EM operator over the specified region. Together, these facts allow us to derive a bound (with high probability) on the error of the sample iterates when the initializer is in the region. Finally, Section~\ref{sec:initialization} introduces a random initialization strategy that is shown to give a large probability to the region for which our error bound applies. The more technical proofs are relegated to the Appendix (Section~\ref{sec:appendix}).

\section{EM iterates}\label{sec:em-iterates}

We will consider the problem of estimating the centers of a two-component spherical Gaussian mixture
\begin{align*}
Y \sim \tfrac{1}{2} N(\ts, \sigma^2 I_d) + \tfrac{1}{2} N(-\ts, \sigma^2 I_d).
\end{align*}
Notice that we require the two component means to sum to zero. Realize that the corresponding model with arbitrary means can be transformed into this form by subtracting the population mean (or approximately transformed by subtracting the sample mean).

The log likelihood of a mixture model is typically difficult to maximize because of the summation inside the logarithm. Expressed in terms of a single observation, it takes the form
\begin{align*}
\log p_\theta (y) = \log \sum_k \lambda_k p_{\theta_k} (y)
\end{align*}
However, the likelihood can be expressed as the marginal likelihood of a joint distribution that includes both the observed data and latent variables corresponding to the component labels. The log likelihood of this joint density can be expressed as a sum of logarithms.
\begin{align*}
\log p_\theta (y, z) &= \log \prod_k [\lambda_k p_{\theta_k} (y)]^{z_k}\\
 &= \sum_k z_k \log \lambda_k p_{\theta_k} (y)
\end{align*}
where the marginal density $p_\theta (z)$ is multi-Bernoulli.

The EM algorithm is a common tool for optimizing the log likelihood when latent variables are present. It proceeds by iteratively maximizing the expected joint log likelihood given the data and current parameter values.
\begin{align*}
\th_{t+1} \leftarrow \argmax_{\theta' \in \Theta} \E_{Z | y, \th_t} \log p_{\theta'} (y, Z)
\end{align*}
In the case of mixture models, the objective function simplifies to
\begin{align*}
\E_{Z | y, \theta} \log p_{\theta'} (y, Z) &= \sum_k \E [Z_k | y, \theta] \log \lambda_k' p_{\theta_k'} (y)
\end{align*}
where both the weights and the components' parameters are encoded in $\theta'$. Because each $Z_k$ is an indicator variable, the expectation is a probability. By Bayes theorem,
\begin{align*}
\E [Z_k | y, \theta]  &= \P [Z_k=1 | Y=y, \theta]\\
 &= \frac{\P [Z_k=1, Y=y | \theta]}{\P [Y=y | \theta]}\\
 &= \frac{\lambda_k p_{\theta_k} (y)}{\sum_j \lambda_j p_{\theta_j} (y)}
\end{align*}
These expectations sum to one.\\

For the simple Gaussian mixture that we will analyze, the expectation of $Z_1$ is
\begin{align*}
\E [Z_1 | y, \theta] &= \frac{e^{-\| y - \theta \|^2/2\sigma^2}}{e^{-\| y - \theta \|^2/2\sigma^2} + e^{-\| y + \theta \|^2/2\sigma^2}}\\
 &= \frac{1}{1 + e^{-2\langle \theta, y \rangle/\sigma^2}}\\
 &= \omega(\tfrac{\langle \theta, y \rangle}{\sigma^2})
\end{align*}
where $\omega$ denotes the [horizontally stretched] logistic function
\begin{align}\label{eq:logistic-definition}
\omega(t) := \frac{1}{1 + e^{-2t}}.
\end{align}
Likewise, the expectation of $Z_2$ is $\omega(-\tfrac{\langle \theta, y \rangle}{\sigma^2})$, which is also $1-\omega(\tfrac{\langle \theta, y \rangle}{\sigma^2})$. Using this identity, we can express the EM algorithm's objective function as
\begin{align*}
Q_y(\theta' | \theta) &:= \sum \E [Z_k | y, \theta] \log \lambda_k' p_{\theta_k'} (y)\\ &= -\tfrac{1}{2} \omega(\tfrac{\langle \theta, y \rangle}{\sigma^2}) \| y - \theta' \|^2 - \tfrac{1}{2}(1-\omega(\tfrac{\langle \theta, y \rangle}{\sigma^2})) \| y + \theta' \|^2\\
 &= - \tfrac{1}{2} \| \theta' \|^2 - (1 - 2 \omega(\tfrac{\langle \theta, y \rangle}{\sigma^2})) \langle \theta', y \rangle - \| y \|^2
\end{align*}

The gradient with respect to the first argument is
\begin{align}\label{eq:q-gradient}
\nabla Q_y(\theta' | \theta) = - \theta' - (1-2\omega(\tfrac{\langle \theta, y \rangle}{\sigma^2}))y
\end{align}
The critical value $2y \omega(\tfrac{\langle \theta, y \rangle}{\sigma^2})) - y$ is the maximizer.

With an iid sample of size $n$, the overall objective function $Q_n$ is simply the sum of the single-observation objective functions. This leads to the update
\begin{align*}
\th_{t+1} \leftarrow M_n(\th_t)
\end{align*}
where the operator mapping from one iteration to the next is
\begin{align*}
M_n(\theta) := \frac{2}{n} \sum y_i \omega (\tfrac{\langle y_i, \theta \rangle}{\sigma^2}) - \frac{1}{n} \sum y_i
\end{align*}
Its population counterpart will be denoted $M$.
\begin{align*}
M(\theta) := 2 \E Y \omega (\tfrac{\langle Y, \theta \rangle}{\sigma^2})
\end{align*}
The population objective function $Q$ is the expectation of $Q_Y$. The true parameter value $\ts$ (or $-\ts$) maximizes $Q$ and is a fixed point of $M$ \cite{McLachlan2008}.\\

Throughout the remainder of this paper, $\phi_\theta$ denotes the density of $N(\theta, \sigma^2 I_d)$, and $f$ is the symmetric mixture $\tfrac{1}{2} \phi_{\ts} + \tfrac{1}{2} \phi_{-\ts}$. We will use $X$, $Y$, and $Z$ to represent generic random variables distributed according to $\phi_{\ts}$, $f$, and $N(0, 1)$ respectively.  We define the ``signal-to-noise ratio" $s := \| \ts \|/\sigma$. We will continue to use $\omega$ to denote the [horizontally stretched] logistic function (\ref{eq:logistic-definition}) and sometimes we use the shorthand
\begin{align*}
\omega_{\theta}(x) := \omega(\tfrac{\langle \theta, x \rangle}{\sigma^2}).
\end{align*}

Additionally, we will make repeated use of the following tail bound for the standard normal variable. 
\begin{align}\label{eq:improved-chernoff}
\P(Z > t) \leq \frac{1}{2}e^{-t^2/2}
\end{align}
for $t \geq 0$. It is one half times the Chernoff bound and can be deduced from Formula 7.1.13 in \cite{Abramowitz1964} via inequality (7) from \cite{Cook2009}.

\section{Iteration error bounds}\label{sec:error-bounds}

Two regions of $\R^d$ will be crucial to our analysis. Define the half-space $\H_a$ and ball $\B_r$ by
\begin{align*}
\H_a := \{ \theta \, | \, \langle \theta, \ts \rangle \geq a \| \ts \|^2 \} \qquad \text{and} \qquad \B_r := \{ \theta \, | \, \| \theta \| \leq r \| \ts \| \}
\end{align*}
where we require $a \in (0, 1)$ and $r \geq 1$. Specifically, we will analyze the behavior of the EM iterations that take place in the intersection of these regions $\D_{a,r} := \H_a \cap \B_r$. (In two-dimensions, this intersection is ``D''-shaped.) Some of the results below are stated for general $a$, but for simplicity, the main analysis considers specifically $a=1/2$.\\

Our essential population result is that $M$ is contractive toward $\ts$ in $\D_{1/2, r}$ as long as $r$ is in a valid range.

\begin{theorem} \label{contractive}
If $ c_1 \leq r \leq c_2s/\sqrt{\log(es)}$, then $\exists \gamma < 1$ such that
\begin{align*}
\| M(\theta) - \ts \| \leq \gamma \| \theta - \ts \|
\end{align*}
for all $\theta \in \D_{1/2,r}$.
\end{theorem}

The proof is in Section~\ref{sec:contractive}, followed by a comparison to the general framework introduced in \cite{Balakrishnan2014}. We show that $ \gamma(s,r) := 76r^4e^{-(1/16)(s/r)^2} $.

Next, we establish that $M$ is stable in regions of the form $\D_{a,r}$ for valid $(a, r)$. In fact, we will need it to be stable with an additional margin that will be used to ensure stability of the sample operator $M_n$ with high probability.

\begin{lemma} \label{inner_prod_bound}
Assume $\theta \in \D_{a,r}$, and let $\kappa_1$ be any number in $(a, 1)$. If $r \leq \frac{as}{\sqrt{5\log(2/(1-a/\kappa_1))}}$, then
\begin{align*}
\langle M(\theta), \ts \rangle \geq (a/\kappa_1) \| \ts \|^2.
\end{align*}
\end{lemma}

\begin{lemma} \label{norm_bound}
Assume $\theta \in \D_{a,r}$, and let $\kappa_2$ be any number in $(0, 1)$. If $\frac{4}{\kappa_2} \leq r \leq \frac{as}{\sqrt{5\log(8/\kappa_2)}}$, then
\begin{align*}
\|M(\theta)\| < \kappa_2 r\|\theta^{\star}\|.
\end{align*}
\end{lemma}

Lemma~\ref{inner_prod_bound} tells us that $M$ stays in $\H_a$, while Lemma~\ref{norm_bound} tells us that $M$ stays in $\B_r$. If $(a, r)$ satisfies the conditions of both Lemmas, then $M$ is stable in $\D_{a, r}$. Note that we need $s$ to be large enough to ensure the existence of valid ranges for $r$.\\

Let $S_{a,r}$ be the least upper bound on the norm of the difference between the sample and population operators in the region $\D_{a,r}$.
\begin{align*}
S_{a,r} := \sup_{\theta \in \D_{a,r}} \| M_n(\theta) - M(\theta) \|
\end{align*}

\begin{lemma} \label{stability}
Suppose $ \kappa_1 $ and $ \kappa_2 $ are as in Lemmata \ref{inner_prod_bound} and \ref{norm_bound} and $ a $ and $ r $ simultaneously satisfy the conditions stated therein. If
\begin{equation*} 
S_{a,r} \leq \|\ts\| \min\{  a(1/\kappa_1-1), r(1-\kappa_2) \}
\end{equation*}
then $M_n$ is stable in $\D_{a, r}$.
\end{lemma}
\begin{proof}
First, note that
\begin{align*}
\inf_{\theta\in \D_{a,r}}\langle M_n(\theta), \theta^{\star} \rangle
& \geq \inf_{\theta\in \D_{a,r}}[\langle M(\theta), \theta^{\star} \rangle - \|M_n(\theta)-M(\theta)\|\|\theta^{\star}\|] \\
& \geq (a/\kappa_1)\|\theta^{\star}\|^2 - a(1/\kappa_1-1)\|\theta^{\star}\|^2 \\
& = a\|\theta^{\star}\|^2,
\end{align*}
where the lower bound on $ \langle M(\theta), \theta^{\star} \rangle $ was proved in Lemma \ref{inner_prod_bound}.
Finally, observe that
\begin{align*}
\sup_{\theta\in \D_{a,r}}\|M_n(\theta)\|
& \leq \sup_{\theta\in \D_{a,r}}[\|M(\theta)\| + \|M_n(\theta)-M(\theta)\|] \\
& \leq r\kappa_2\|\theta^{\star}\| + r(1-\kappa_2)\|\theta^{\star}\| \\
& = r\|\theta^{\star}\|,
\end{align*}
where the upper bound on $ \|M(\theta)\| $ was proved in Lemma \ref{norm_bound}.

\end{proof}

\begin{lemma} \label{sup_norm_bound}
If $n \geq c_3 d \log (1/\delta)$, then
\begin{align*}
S_{a,r} \leq c_4 r \| \ts \| \sqrt{\| \ts \|^2 + \sigma^2} \sqrt{\frac{d \log (1/\delta)}{n}}
\end{align*}
with probability at least $1-\delta$.
\end{lemma}
\begin{proof}
The proof is almost identical to Corollary 2 in \cite{Balakrishnan2014}. It uses a standard discretization and Hoeffding moment generating function argument to bound $ S_{a,r} $. The only difference here is that we control the supremum over $ \D_{a,r} $ instead of a Euclidean ball.
\end{proof}

Combining the conditions of Lemmas~\ref{stability} and~\ref{sup_norm_bound}, and specializing to the $a=1/2$ case, we define
\begin{align*}
N_\delta := d \log (1/\delta) \max \left\{ c_3, \frac{c_4^2 r^2 (\| \ts \|^2 + \sigma^2)}{[\min\{ (1/\kappa_1-1)/2, r(1-\kappa_2)]^2} \right\}
\end{align*}
One can verify that if $n \geq N_\delta$, then the bound in Lemma~\ref{sup_norm_bound} is no greater than the bound in Lemma~\ref{stability}. Thus if $n \geq N_\delta$, then $S_{1/2, r}$ satisfies both bounds with probability at least $1-\delta$.\\

\begin{theorem} \label{main}
If $\tn \in \D_{1/2,r}$, $ c_1 \leq r \leq c_2s/\sqrt{\log(es)} $, and $n \geq N_\delta$,  
then the EM iterates $\{ \th_t \}_{t=0}^\infty$ satisfy the bound
\begin{equation} \label{bound}
\| \th_t - \ts \| \leq \gamma^t \| \tn - \ts \| + \frac{1}{1-\gamma}c_4 r \| \ts \| \sqrt{\| \ts \|^2 + \sigma^2} \sqrt{\frac{d \log (1/\delta)}{n}}
\end{equation}
with probability at least $1-\delta$.
\end{theorem}

\begin{proof}
By Lemma \ref{stability}, the empirical EM iterates $\{ \th_t \}_{t=0}^\infty$ all belong to $ \D_{1/2,r} $ with probability at least $ 1-\delta $. Note that the prescribed constants $ c_1 $ and $ c_2 $ depend on $ \kappa_1 $ and $ \kappa_2 $. We will show that
\begin{equation*}
\| \th_t - \ts \| \leq \gamma^t \| \tn - \ts \| + \sum_{k=0}^{t-1}\gamma^k S_{1/2,r},
\end{equation*}
with probability at least $ 1- \delta $.
To this end, suppose the previous bound holds.
Then
\begin{align*}
\| \th_{t+1} - \ts \| & = \| M_n(\th_t) - \ts \| \\
& \leq \| M(\th_t) - \ts \| + \| M_n(\th_t) - M(\th_t) \| \\
& \leq \| M(\th_t) - \ts \| + S_{1/2,r}\\
 &\leq \gamma\|\th_t - \ts \| + S_{1/2,r} \qquad \text{by Lemma~\ref{contractive}}\\
 &\leq \gamma \left[\gamma^t \| \tn - \ts \| + \sum_{k=0}^{t-1}\gamma^k S_{1/2,r}\right]\\
 &= \gamma^{t+1}\| \tn - \ts \| + \sum_{k=0}^{t}\gamma^k S_{1/2,r}
\end{align*}
which confirms the inductive step. The $t=1$ case uses the same reasoning.\\

The theorem then follows from the fact that $ \sum_{k=0}^{t}\gamma^k \leq 1/(1-\gamma) $ and the bound on $ S_{1/2,r} $ from Lemma \ref{sup_norm_bound}.
\end{proof}

\begin{remark}
The fact that $ c_1 \leq r \leq c_2s/\sqrt{\log s} $ was determined from the conditions in Lemmata \ref{inner_prod_bound} and \ref{norm_bound} and Theorem \ref{contractive}. To reiterate we need
\begin{itemize}
\item $ s > 4r\sqrt{\log(76r^4)} $
\item $ \frac{4}{\kappa_2} \leq r \leq \frac{as}{\sqrt{5\log(8/\kappa_2)}} $
\item $ r \leq \frac{as}{\sqrt{5\log(2/(1-a/\kappa_1))}} $
\end{itemize}
to hold simultaneously. We also require that $ a $ belong to $ (0,1) $, $ \kappa_1 $ belong to $ (a,1) $, and $ \kappa_2 $ belong to $ (0,1) $. As a concrete example, with $ a = 1/2 $ and $ \kappa_1 = \kappa_2 = 3/4 $, all conditions are satisfied if $ 6 \leq r \leq s/(8\sqrt{\log(es)}) $.

\end{remark}

\section{Initialization strategy}\label{sec:initialization}

Theorem~\ref{main} describes the behavior of the EM iterates if the initialization is in a desirable region of the form $\D_{1/2,r}$. Realize, however, that by symmetry it is just as good to initialize in the corresponding region for $-\ts$. Thus, we define
\begin{align*}
\tilde{\H}_a := \H_a \cup - \H_a = \{ | \langle \tn, \ts \rangle | \geq \| \ts \|^2/2 \} \qquad \text{and} \qquad \tilde{\D}_{a,r} := \tilde{\H}_a \cap \B_r
\end{align*}
See Figure~\ref{fig}. Estimates $\hat{\theta}$ and $-\hat{\theta}$ correspond to the same mixture distribution in this model. We should interpret the results from Section~\ref{sec:error-bounds} in terms of distributions and thus not distinguish between estimating $\ts$ and estimating $-\ts$.
\begin{figure}
  \centering
  \includegraphics[width=9cm, height=9cm]{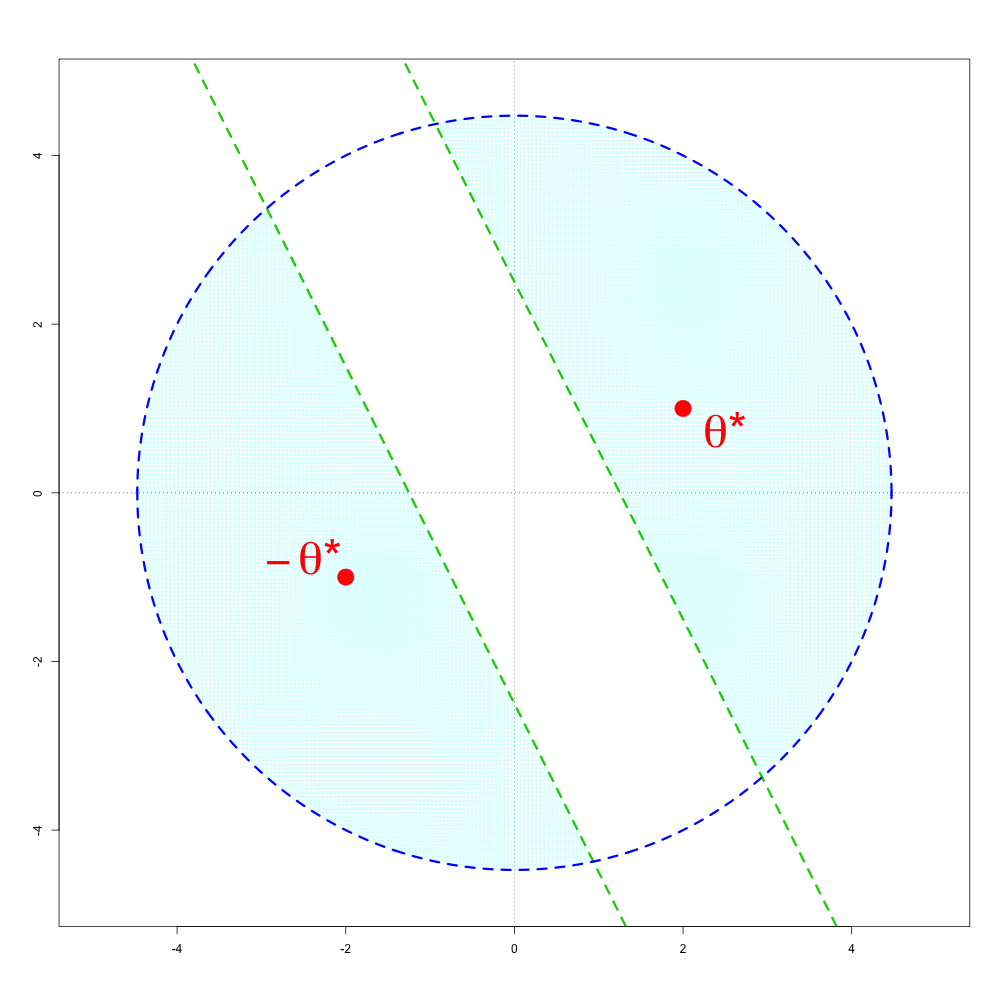}
  \caption{An example region $ \tilde{\D}_{a,r} $ in two dimensions.}
  \label{fig}
\end{figure}
Our error bounds in the previous section are conditional on the initializer being in the specified region, but we have yet to discuss how to generate such an initializer. As a first thought, note that initializing EM with the method of moments estimator has been shown to perform well in simulations \cite{pereira2012}. Furthermore, tensor methods have recently been devised for finding the method of moments estimator for Gaussian mixtures \cite{Anandkumar2014}. It would be interesting to analyze the behavior of that strategy with respect to $\D_{a, r}$. However, here we instead opt for a random initialization strategy for which we can derive a straight-forward lower bound on the probability of starting in $\tilde{\D}_{a,r}$.\\


For the remainder of this section, $\tilde{\D}_{a,r}$ will be considered a random event. For the first result, we will pretend that $\| \ts \|$ is known and can thus be used in the initialization.
\begin{proposition}\label{thm:initialization}
Let $\tn \sim N(0, \| \ts \|^2 I_d)$. Then
\begin{align}\label{initializer-probability}
\P (\tilde{\D}_{a,r}) \geq 2 \Phi(-a)-\P \left( \chi_d^2 > r^2 \right)
\end{align}
where $\Phi$ is the standard Normal cdf.
\end{proposition}

\begin{proof}
The probability of the intersection of $\tilde{\H}_a$ and $\B_r$ has a simple bound in terms of the complement of $\B_r$.
\begin{align*}
\P(\tilde{\D}_{a,r}) &= \P (\tilde{\H}_a \cap \B_r)\\
&= \P \tilde{\H}_a - \P (\tilde{\H}_a \cap \B_r^c)\\
&\geq \P \tilde{\H}_a - \P \B_r^c
\end{align*}

First, consider the event $\tilde{\H}_a$.
\begin{align*}
\P(\tilde{\H}_a) &= \P ( | \langle \tn, \ts \rangle | \geq a \| \ts \|^2)\\
 &= 2 \P \left( \left\langle \tfrac{\tn}{\| \ts \|}, \tfrac{\ts}{\| \ts \|} \right\rangle \geq a \right)\\
 &= 2 \P (Z \geq a)
\end{align*}
where $Z$ is standard Normal.\\

For the complement of $\B_r$,
\begin{align*}
\P(B^c_r) &= \P (\| \tn \| > r \| \ts \|)\\
 &= \P \left( \left\| \tfrac{\tn}{\| \ts \|} \right\| > r \right)\\
 &= \P \left( \left\| \tfrac{\tn}{\| \ts \|} \right\|^2 > r^2 \right)\\
 & = \P \left( \chi_d^2 > r^2 \right).
\end{align*}
\end{proof}

Proposition~\ref{thm:initialization} is for initializing with a known $\| \ts \|^2$. In practice, this quantity can be estimated from the data by
\begin{equation*}
\hat{T} := \frac{1}{n}\sum_i(\|Y_i\|^2-d\sigma^2).
\end{equation*}
In fact, $ \hat{T} $ can be shown to concentrate around $ \|\theta^{\star}\|^2 $ with high probability, as we will show. This gives an intuitive rationale to instead sample $ \tn $ from a $ N(0, (\hat{T}_++\epsilon)I_d) $ distribution (where $ \epsilon $ is a positive number).

\begin{proposition} Suppose $ \hat{\theta}_0 $ follows a $ N(0, (\hat{T}_++\sigma^2/2)I_d) $. Then
\begin{equation} \label{initializer-probability-estimator}
\P (\tilde{\D}_{a,r}) \geq [2 \Phi(-a) -\P (\chi^2_d > r^2/2 )]\P(E),
\end{equation}
where $ E = \{ |\hat{T}-\|\ts\|^2| < \sigma^2/2 \} $.
\end{proposition}

\begin{proof}
First, note that
\begin{equation*}
\P (\tilde{\D}_{a,r}) \geq \P (\tilde{\D}_{a,r}\cap E) \geq \P (\tilde{\H}_a \cap E) - \P(\B_r^c \cap E).
\end{equation*}
On $ E $, $ \|\ts\|^2 \leq \hat{T}_++\sigma^2/2 $ and hence
\begin{equation} \label{event1}
\left\{ \left| \left\langle \tfrac{\tn}{\sqrt{\hat{T}_++\sigma^2/2}}, \tfrac{\ts}{\|\ts\|} \right\rangle \right| \geq a \right\} \cap E
\end{equation}
is contained in $ \tilde{\H}_a \cap E $.

Since $ s \geq 1 $, $ \sigma^2/2 \leq \|\ts\|^2/2 $ and hence on $ E $, $ \hat{T}_++\sigma^2/2 \leq 2\|\ts\|^2 $. Thus the event
\begin{equation} \label{event2}
\left\{ \left\| \tfrac{\tn}{\sqrt{\hat{T}_++\sigma^2/2}} \right\|^2 > r^2/2 \right\} \cap E
\end{equation}
contains $ \B_r^c \cap E $.
The final result follows by integrating the indicator variables of \eqref{event1} and  \eqref{event2} with respect to the the joint distribution of $ \tn $ and $ \hat{T} $ and then finally integrating with respect to the distribution of $ \hat{T} $.
\end{proof}

\begin{remark}
By the Chernoff tail bound for a $ \chi^2_d $ random variable, $ \P(\chi^2_d > r^2) \leq (r/\sqrt{d})^de^{-(r^2-d)/2} $. Thus, the condition $r > \sqrt{2d} $ is necessary for \eqref{initializer-probability-estimator} to be positive. By Theorem~\ref{main}, $ s > cr\sqrt{\log r} $ for the bound \eqref{bound} to hold. Thus if the signal to noise ratio is at least a constant multiple of $ \sqrt{d\log d} $, there is some $q > 0$ that lower bounds the probability that a given initializer $ \tn $ is in $ \tilde{\D}_{1/2,r} $ and hence for which $ \eqref{bound} $ holds. By drawing $m$ such initializers independently, the probability is at least $1-(1-q)^m$ that one or more are in $ \tilde{\D}_{1/2,r} $.
\end{remark}

$\P E_{\epsilon}$ can be bounded using Chebychev or Cantelli concentration inequalities, because $\hat{T}$ has variance $2\sigma^2(d+2\|\ts\|^2)/n$. However, Proposition~\ref{concentration} establishes a concentration inequality that decays exponentially with $n$.

\begin{proposition} \label{concentration}
If $ s \geq 1 $ and $ \epsilon < 5d\sigma\|\theta^{\star}\| $, then
\begin{equation*}
 \mathbb{P}(|\hat{T}-\|\theta^{\star}\|^2| > \epsilon) \leq 2\exp\{-n\epsilon^2/(36d\sigma^2\|\theta^{\star}\|^2)\}.
\end{equation*}
\end{proposition}

\section{Appendix}\label{sec:appendix}

\subsection{Stein's lemma for mixtures}\label{sec:stein}

Let $W \sim \sum \lambda_j \phi_{\theta_j}$ be a mixture of spherical Gaussians and $X_j \sim \phi_{\theta_j}$ have the component distributions. A mixture version of Stein's lemma (Lemma 2 in \cite{Stein1981}) holds when $W$ is multiplied by a differentiable function $g$.
\begin{align*}
\E W g(W) &= \int \left[ w g(w) \sum \lambda_j \phi_{\theta_j} \right] dw\\
 &= \sum \lambda_j \E X_j g(X_j)\\
 &= \sum \lambda_j [\E \nabla g(X_j) + \theta_j \E g(X_j)]
\end{align*}

In our present case, $M(\theta)$ is a particularly simple version of this because $Y$ is a symmetric mixture, and $\omega$ is within a constant of an odd function: $\omega(-t) = 1 - \omega(t)$. Let $X$ and $X'$ have the component distributions $\phi_{\ts}$ and $\phi_{-\ts}$.
\begin{align*}
\tfrac{1}{2} M(\theta) &:= \E Y \omega(\tfrac{\langle \theta, Y \rangle}{\sigma^2})\\
 &= \tfrac{1}{2} \E X \omega(\tfrac{\langle \theta, X \rangle}{\sigma^2}) + \tfrac{1}{2} \E X' \omega(\tfrac{\langle \theta, X' \rangle}{\sigma^2})\\
 &= \tfrac{1}{2} \E X \omega(\tfrac{\langle \theta, X \rangle}{\sigma^2}) + \tfrac{1}{2} \E (-X) \omega(\tfrac{\langle \theta, (-X) \rangle}{\sigma^2})\\
 &= \E X \omega(\tfrac{\langle \theta, X \rangle}{\sigma^2}) - \tfrac{1}{2} \ts\\
 &= \E \nabla \omega(\tfrac{\langle \theta, X \rangle}{\sigma^2}) + \ts \E \omega(\tfrac{\langle \theta, X \rangle}{\sigma^2}) - \tfrac{1}{2} \ts\\
 &= \theta \E \omega'(\tfrac{\langle \theta, X \rangle}{\sigma^2}) + \ts [\E \omega(\tfrac{\langle \theta, X \rangle}{\sigma^2}) - \tfrac{1}{2}].
\end{align*}

\subsection{Expectation of a sigmoid}\label{sec:sigmoids}

First, we are interested in the behavior of quantities of the form $\E \psi(\alpha Z + \beta)$ as $\alpha$ and $\beta$ change. Observe that if $\psi$ is any increasing function, then clearly $\E \psi(\alpha Z + \beta)$ is increasing in $\beta$ regardless of the distribution of $Z$. We will next consider how the expectation changes in $\alpha$ in special cases.\\

Throughout the remainder of this section, assume $\psi$ is within a constant of an odd function and that it is twice differentiable, increasing, and concave on $\R^+$. Sigmoids, for instance, typically meet these criteria.
\begin{lemma}\label{lemma:sigmoid}
Let $Z \sim N(0, 1)$. The function $\alpha \mapsto \E \psi(\alpha Z + \beta)$ is non-increasing for $\alpha \geq 0$.
\end{lemma}

\begin{proof}
We will interchange an integral and derivative (justified below), then appeal to Stein's lemma. Also, note that $\psi''$ is an odd function. Let $\phi$ denote the standard normal density.
\begin{align*}
\tfrac{d}{d a} \E \psi(a Z + \beta) |_{a = \alpha} &= \E Z \omega'(\alpha Z + \beta)\\
 &= \alpha \E \psi''(\alpha Z + \beta)\\
 &= \alpha \int \psi''(\alpha z + \beta) \phi(z) dz\\
 &= \int \psi''(u) \phi(\tfrac{u-\beta}{\alpha}) du\\
 &= \int_{u < 0} \psi''(u) \phi(\tfrac{u-\beta}{\alpha}) du + \int_{u \geq 0} \psi''(u) \phi(\tfrac{u-\beta}{\alpha}) du\\
 &= \int_{u > 0} \psi''(-u) \phi(\tfrac{-(u)-\beta}{\alpha}) du + \int_{u \geq 0} \psi''(u) \phi(\tfrac{u-\beta}{\alpha}) du\\
 &= - \int_{u > 0} \psi''(u) \phi(\tfrac{u+\beta}{\alpha}) du + \int_{u \geq 0} \psi''(u) \phi(\tfrac{u-\beta}{\alpha}) du\\
 &= \int_{u \geq 0} \psi''(u) [\phi(\tfrac{u-\beta}{\alpha}) - \phi(\tfrac{u+\beta}{\alpha})] du.
\end{align*}
Because $\psi$ is concave on $\R^+$, it's second derivative is negative. The other factor is non-negative on $\R^+$, so the overall integral is negative.\\

We still need to justify the interchange. First, use the fundamental theorem of calculus to expand $\psi(\alpha z + \beta)$ inside an integral over $\R^+$. Because $\phi'$ is non-negative, Tonelli's theorem justifies the change of order of integration. Then take a derivative of both sides.
\begin{align*}
\int_0^\infty \psi(\alpha z + \beta) \phi(z) dz &= \int_0^\infty \left[\psi((0) z + \beta) + \int_0^\alpha \tfrac{\partial}{\partial a} \psi(a z + \beta)\right] \phi(z) dz\\
 &= \int_0^\infty \psi(\beta) \phi(z) dz + \int_0^\alpha \int_0^\infty z \psi'(a z + \beta) \phi(z) dz
\end{align*}
\begin{align*}
\Rightarrow \qquad \frac{d}{d a} \left(\int_0^\infty \psi(a z + \beta) \phi(z) dz \right)_{a = \alpha} &= \int_0^\infty z \psi'(\alpha z + \beta) \phi(z) dz
\end{align*}
Tonelli's theorem justifies the interchange for the integral over $\R^-$ as well. Use the fact that the derivative of the sum is the sum of the derivatives to put everything back together.
\end{proof}

\begin{remark}
By symmetry, of course, $\alpha \mapsto \E \psi(\alpha Z + \beta)$ is non-decreasing for $\alpha \leq 0$, which tells us that $\E \psi(\alpha Z + \beta) \leq \psi(\beta)$.
\end{remark}

\begin{remark}
This result actually holds for any Normal random variable. Indeed, because any Normal $X$ can be expressed as $\alpha Z + \beta$, we see that $\E \psi(X)$ is increasing in the variance of $X$.
\end{remark}


\begin{remark}
The [stretched] logistic function $\omega$ satisfies the criteria for Lemma~\ref{lemma:sigmoid}.
\end{remark}

\begin{corollary}\label{cor:psi-lower-bound}
Let $Z \sim N(0, 1)$ and $\beta \geq 0$. Then
\begin{align*}
\E \psi(\alpha Z + \beta) \geq \psi(0).
\end{align*}
\end{corollary}

\begin{proof}
We know that the minimizing [non-negative] value of $\beta$ is $0$. According to our derivation in Lemma~\ref{lemma:sigmoid}, when $\beta = 0$ the derivative of $\alpha \mapsto \E \psi(\alpha Z + \beta)$ is zero everywhere. That is, the expectation is the same at every $\alpha$; evaluating at $\alpha = 0$ gives the desired result.
\end{proof}

We will also need lower bounds on the expectation of $\omega$. First, we establish a more general fact for sigmoids.

\begin{lemma}\label{lemma:sigmoid-lower-bound}
If $\rho$ is a positive non-decreasing function and $Z \sim N(0, 1)$, then for any $q \geq 0$,
\begin{align*}
\E \rho(\alpha Z + \beta) \geq \rho(\beta - q) (1 - \tfrac{1}{2} e^{-q^2/2\alpha^2})
\end{align*}
\end{lemma}

\begin{proof}
By Markov's inequality
\begin{align}\label{eq:rho-tail-bound}
\P (\alpha Z + \beta > t) &\leq \P (\rho(\alpha Z + \beta) \geq \rho(t))\nonumber\\
&\leq \frac{\E \rho(\alpha Z + \beta)}{\rho(t)}
\end{align}

Using the Gaussian tail bound (\ref{eq:improved-chernoff}),
\begin{align*}
\P (\alpha Z + \beta > t) &= \P (Z > \tfrac{t - \beta}{\alpha})\\
 &= 1- \P (Z \leq \tfrac{t - \beta}{\alpha})\\
 &\geq 1 - \tfrac{1}{2} e^{-(t-\beta)^2/2\alpha^2}
\end{align*}
as long as $t \leq \beta$. Putting this together with (\ref{eq:rho-tail-bound}), and setting $t := \beta - q$ gives the lemma.
\end{proof}

Recall that we defined $s$ to be the signal-to-noise ratio $\| \ts \|/\sigma$.

\begin{lemma}\label{lemma:omega-lower-bound}
If $\theta \in \D_{a, r}$ and $X \sim N(\ts, \sigma^2 I_d)$, then
\begin{align*}
\E \omega(\tfrac{\langle \theta, X \rangle}{\sigma^2}) > 1 - e^{-(as/r)^2/5}
\end{align*}
\end{lemma}

\begin{proof}
First, realize that we can write $X$ as a transformation of a $d$-dimensional standard normal: $\sigma Z_d + \ts$. The inner product of $Z_d$ with any unit vector has a one-dimensional standard normal. We can also use the assumptions that $\| \theta \| \leq r \| \ts \|$ and $\langle \theta, \ts \rangle \geq a \| \ts \|^2$ along with the monotonicity properties of $\E \omega(\alpha Z + \beta)$ derived above.
\begin{align}\label{eq:omega-expression-ars}
\E \omega(\tfrac{\langle \theta, X \rangle}{\sigma^2}) &= \E \omega(\tfrac{\| \theta \|}{\sigma} Z + \tfrac{\langle \theta, \ts \rangle}{\sigma^2})\nonumber\\
 &\geq \E \omega(\tfrac{r \| \ts \|}{\sigma} Z + \tfrac{a \| \ts \|^2}{\sigma^2})\nonumber\\
 &= \E \omega(rs Z + as^2).
\end{align}

Let's specialize Lemma~\ref{lemma:sigmoid-lower-bound} to a particular claim for $\omega$.
\begin{align}\label{eq:omega-expectation-lower-bound}
\E \omega(\alpha Z + \beta)
& \geq \sup_{t \leq \beta}\left\{\frac{1 - e^{-(\beta-t)^2/2\alpha^2}}{1+e^{-2t}}\right\}\nonumber\\
& = \sup_{t \leq \beta}\left\{\frac{(1 - e^{-(\beta-t)^2/4\alpha^2})(1 + e^{-(\beta-t)^2/4\alpha^2})}{1+e^{-2t}}\right\}\nonumber\\
& \geq 1-e^{-2 t_0},
\end{align}
where $t_0 \leq \beta$ is a solution to the quadratic equation $2 t_0 = (\beta-t_0)^2/4\alpha^2$. Notice that when this equation is satisfied, the last step of the derivation follows by canceling the denominator with the right-hand factor of the numerator. A solution to this quadratic is
\begin{align*}
t_0 & = \beta + 4\alpha^2(1-\sqrt{\beta/(2\alpha^2) + 1}) \\
& = \frac{(\beta/\alpha)^2/2}{(1+\sqrt{\beta/(2\alpha^2)+1})^2}.
\end{align*}
The first expression shows that this $t_0$ is less than $\beta$. The second clarifies the relationships we'll need between $\alpha$ and $\beta$ and shows that $t_0$ is also non-negative.\\

Applying this bound to (\ref{eq:omega-expression-ars}), we have
\begin{align*}
t_0 &= \frac{(as/r)^2/2}{(1+\sqrt{a/(2r^2)+1})^2}\\
 &> (as/r)^2/10.
\end{align*}
The last step comes from upper bounding the denominator by $5$. (Recall that we require $a \in [0, 1]$ and $r \geq 1$.)
\end{proof}

\begin{lemma}\label{lemma:normal-difference-integral}
Let $\rho$ be any bounded and twice-differentiable Lipschitz function, and let $X_0 \sim N(\mu_0, \sigma_0)$ and $X_1 \sim N(\mu_1, \sigma_1)$. Then
\begin{align*}
\E \rho(X_1) - \E \rho(X_0) &= \int_0^1 \E [(\mu_1 - \mu_0) \rho'(X_\lambda) + \tfrac{1}{2} (\sigma_1^2 - \sigma_0^2) \rho''(X_\lambda)] d \lambda
\end{align*}
where $X_\lambda \sim (1-\lambda) N(\mu_0, \sigma_0) + \lambda N(\mu_1, \sigma_1)$.
\end{lemma}

\begin{proof}
This is a variant of Theorem 2 in \cite{Muller2001}, which presents the result in $d$ dimensions and with much weaker regularity conditions.
\end{proof}

\begin{lemma}\label{lemma:omega-prime-expectation-bound}
Suppose $ |\mu| \leq 2\sigma^2  $. Then $ \mathbb{E}\omega^{\prime}(\sigma Z + \mu) \leq 2e^{-(1/2)(\mu/\sigma)^2} $.
\end{lemma}
\begin{proof}
Note that $ \omega^{\prime}(t) \leq 2e^{-2|t|} $. Thus
\begin{align*}
\omega^{\prime}(\sigma z + \mu)\phi(z) 
& \leq
2e^{-2|\sigma z + \mu|}\phi(z) \\
& = 2\mathbb{I}\{z > -\mu/\sigma\}e^{2(\sigma^2-\mu)}\phi(z+2\sigma) + \\ & \qquad\qquad 2\mathbb{I}\{z < -\mu/\sigma\}e^{2(\sigma^2+\mu)}\phi(z-2\sigma),
\end{align*}
where the last line follows from completing the square.
Next, integrate both sides of the inequality over $ \mathbb{R} $, making the change of variables $ u = z+2\sigma $ and $ u = z-2\sigma $ on each region of integration. This leads to the upper bound
\begin{equation*}
2e^{2(\sigma^2 - \mu)} \mathbb{P}\left(Z > 2\sigma - \mu/\sigma\right) + 2e^{2(\sigma^2 + \mu)} \mathbb{P}\left(Z > 2\sigma + \mu/\sigma\right).
\end{equation*}
Next, use the fact that $ \mathbb{P}(Z > t) \leq \frac{1}{2}e^{-t^2/2} $ for all $ t \geq 0 $. Since $ |\mu| \leq 2\sigma^2 $, we have that $ 2\sigma \pm \mu/\sigma \geq 0 $. Plugging in $ t = 2\sigma \pm \mu/\sigma $ and performing some algebra proves the result.
\end{proof}

\begin{lemma}\label{lemma:omega-derivatives-inequality}
$| \omega'' | \leq 2 \omega'$ and $| \omega''' | \leq 4 \omega'$.
\end{lemma}

\begin{proof}
Using the relationship $ \omega' = 2\omega(1-\omega) $, one can easily derive the identities
\begin{equation*}
\omega'' = 2\omega'(1-2\omega)
\end{equation*}
and
\begin{equation*}
\omega''' = 4\omega'(1-6\omega+6\omega^2).
\end{equation*}
The fact that $ 0 \leq \omega \leq 1 $ implies $ |1-2\omega| $ and $ |1-6\omega+6\omega^2| $ are both less than one.
\end{proof}

\subsection{Stability of population iterates in $\D_{a, r}$}

\begin{proof}[Proof of Lemma~\ref{inner_prod_bound}]
First, recall the expression for $M(\theta)$ derived in Section~\ref{sec:stein}.
\begin{align*}
\langle M(\theta), \ts \rangle &= 2 \| \ts \|^2 [\E \omega(\tfrac{\langle \theta, X \rangle}{\sigma^2}) - \tfrac{1}{2}] + 2 \langle \theta, \ts \rangle \E \omega'(\tfrac{\langle \theta, X \rangle}{\sigma^2})\\
 &\geq 2 \| \ts \|^2 [\E \omega(\tfrac{\langle \theta, X \rangle}{\sigma^2}) - \tfrac{1}{2}]\\
 &\geq 2 \| \ts \|^2 [(1 - e^{-(as/r)^2/5}) - \tfrac{1}{2}]\\
 &= \| \ts \|^2 (1 - 2 e^{-(as/r)^2/5}).
\end{align*}
We used non-negativity of $\omega'$ and our assumption about $\langle \theta, \ts \rangle$, then we invoked Lemma~\ref{lemma:omega-lower-bound}.\\

The assumed upper bound for $r$ implies that
\begin{align*}
1 - 2 e^{-(as/r)^2/5} \geq a/\kappa_1.
\end{align*}

\end{proof}

\begin{proof}[Proof of Lemma~\ref{norm_bound}]
Again, recall the expression for $M(\theta)$ derived in Section~\ref{sec:stein}. We will use the facts that $\omega' \geq 0$ and $\mathbb{E}\omega(\tfrac{\langle \theta, X \rangle}{\sigma^2}) \geq \omega(0) = 1/2$ (see Corollary~\ref{cor:psi-lower-bound}) when we use the triangle inequality. We will also use the identity $\omega^{\prime} = 2\omega(1-\omega)$.
\begin{align}\label{eq:m-norm-bound}
\|M(\theta)\|
& =  \| 2\theta^{\star} \left( \E \omega(\tfrac{\langle \theta, X \rangle}{\sigma^2}) - 1/2 \right) + 2\theta \E \omega^{\prime}(\tfrac{\langle \theta, X \rangle}{\sigma^2}) \| \nonumber\\
& \leq \|\theta^{\star}\| \left( 2\E \omega(\tfrac{\langle \theta, X \rangle}{\sigma^2}) - 1 \right) + 2\|\theta\| \E \omega^{\prime}(\tfrac{\langle \theta, X \rangle}{\sigma^2}) \nonumber\\
& \leq \|\theta^{\star}\| \left( 2\E \omega(\tfrac{\langle \theta, X \rangle}{\sigma^2}) - 1 \right) + 2r\|\theta^{\star}\| \E \omega^{\prime}(\tfrac{\langle \theta, X \rangle}{\sigma^2}) \nonumber\\
& \leq  \|\theta^{\star} \| [ 2(1+2r)\E \omega(\tfrac{\langle \theta, X \rangle}{\sigma^2}) - 4r\E \omega^2(\tfrac{\langle \theta, X \rangle}{\sigma^2}) - 1 ]\nonumber\\
&\leq \|\theta^{\star} \| [- 4r[\E \omega(\tfrac{\langle \theta, X \rangle}{\sigma^2})]^2 + 2(1+2r)\E \omega(\tfrac{\langle \theta, X \rangle}{\sigma^2}) - 1 ]
\end{align}
where the last step follows from Jensen's inequality.\\

We need to show that the quadratic factor of (\ref{eq:m-norm-bound}) is bounded by $\kappa_2 r$. According to the quadratic theorem, this is true when
\begin{align*}
\E \omega(\tfrac{\langle \theta, X \rangle}{\sigma^2}) \geq \frac{1 + 1/(2r) + \sqrt{1/(4r^2) + 1-\kappa_2}}{2}
\end{align*}
(The other solutions are less than $1/2$ and thus impossible.) Because square root is subadditive, it is sufficient to show that
\begin{align}\label{eq:norm-bound-omega-sufficient}
\E \omega(\tfrac{\langle \theta, X \rangle}{\sigma^2}) \geq \frac{1 + 1/r + \sqrt{1 - \kappa_2}}{2}.
\end{align}

Consider upper bounds for $r$ of the form
\begin{align*}
r \leq \frac{as}{\sqrt{5 \log(2/(g(\kappa_2) - \sqrt{1-\kappa_2}))}}
\end{align*}
where $g$ is any function greater than $\sqrt{1-\kappa_2}$ for $\kappa_2 \in (0, 1]$. Invoking Lemma~\ref{lemma:omega-lower-bound} and substituting this form of upper bound for $r$,
\begin{align*}
\E \omega(\tfrac{\langle \theta, X \rangle}{\sigma^2}) &> 1 - e^{-(as/r)^2/5}\\
 &\geq 1 - \frac{g(\kappa_2) - \sqrt{1-\kappa_2}}{2}
\end{align*}
Comparing this to (\ref{eq:norm-bound-omega-sufficient}), we find that $r$ needs to be at least $\tfrac{1}{1-g(\kappa_2)}$.

If $g(\kappa_2)$ is too close to $\sqrt{1 - \kappa_2}$ near $\kappa_2 = 1$, then the upper bound is too small; but the looser it is, the larger the lower bound is. The result in this lemma takes $g(\kappa_2) := 1 - \kappa_2/4$. For the upper bound, note that
\begin{align*}
g(\kappa_2) - \sqrt{1 - \kappa_2} &= 1 - \kappa_2/4 - \sqrt{1 - \kappa_2}\\
 &\geq 1 - \kappa_2/4 - (1 - \kappa_2/2)\\
 &= \kappa_2/4.
\end{align*}

\end{proof}

\subsection{Contractivity and Discussion}\label{sec:contractive}

\begin{proof}[Proof of Theorem~\ref{contractive}] First, observe that $\ts = M(\ts)$, as pointed out in Section~\ref{sec:em-iterates}. As in Section~\ref{sec:stein}, we can use $ \omega(t) = 1 - \omega(-t) $ and let $X \sim N(\ts, \sigma^2 I_d)$ to obtain a more manageable expression.
\begin{align*}
\tfrac{1}{2}[M(\theta) - M(\ts)] &= \E Y \left[\omega(\tfrac{\langle \theta, Y \rangle}{\sigma^2}) - \omega(\tfrac{\langle \ts, Y \rangle}{\sigma^2})\right]\\
&= \E X \left[\omega(\tfrac{\langle \theta, X \rangle}{\sigma^2}) - \omega(\tfrac{\langle \ts, X \rangle}{\sigma^2})\right]\\
 &= \E [X \Delta \omega_{\theta}(X)]
\end{align*}
where $\Delta \omega_{\theta}(X)$ denotes the difference $\omega(\tfrac{\langle \theta, X \rangle}{\sigma^2}) - \omega(\tfrac{\langle \ts, X \rangle}{\sigma^2})$.

By Stein's lemma,
\begin{align}\label{eq:expectation-x-delta}
\E [X \Delta \omega_{\theta}(X)] &= \ts \E \left[\omega(\tfrac{\langle \theta, X \rangle}{\sigma^2}) - \omega(\tfrac{\langle \ts, X \rangle}{\sigma^2})\right]\nonumber\\
 & \qquad + \E \left[\theta \omega'(\tfrac{\langle \theta, X \rangle}{\sigma^2}) - \ts \omega'(\tfrac{\langle \ts, X \rangle}{\sigma^2})\right]\nonumber\\
 &= \ts \E\Delta\omega_{\theta}(X) + \theta \E\Delta\omega_{\theta}'(X) + (\theta - \ts)\E\omega'(\tfrac{\langle \ts, X \rangle}{\sigma^2}).
\end{align}

Using Lemma~\ref{lemma:normal-difference-integral}, we can express the expectation in the first term of (\ref{eq:expectation-x-delta}) as
\begin{align*}
\E\Delta\omega_{\theta}(X)
& =  \int_0^1 \E \left[ (\mu_1 - \mu_0) \omega'\left(\sigma_{\lambda}Z+\mu_{\lambda} \right) + \dfrac{\sigma_1^2 - \sigma_0^2}{2} \omega''\left(\sigma_{\lambda}Z+\mu_{\lambda}\right) \right] d\lambda
\end{align*}
where $ \mu_{\lambda} := (1-\lambda) \tfrac{\|\ts \|^2}{\sigma^2} + \lambda \tfrac{\langle \theta, \ts \rangle}{\sigma^2}$, and $\sigma^2_{\lambda} := (1-\lambda) \tfrac{\|\ts \|^2}{\sigma^2} + \lambda \tfrac{\|\theta\|^2}{\sigma^2}$. We can bound the sizes of the coefficients of $\omega'$ and $\omega''$ as follows.
\begin{align*}
|\mu_1 - \mu_0| &= \left|\dfrac{\langle \theta^{\star}, \theta \rangle - \|\theta^{\star} \|^2}{\sigma^2}\right|\\
 &\leq \dfrac{\|\theta^{\star}\|\|\theta - \theta^{\star}\|}{\sigma^2}
\end{align*}
and
\begin{align*}
|\sigma_1^2 - \sigma_0^2| &= \left|\dfrac{\|\theta\|^2 - \|\theta^{\star} \|^2}{\sigma^2}\right|\\
 &\leq \dfrac{(\|\theta\| + \|\theta^{\star}\|)\|\theta - \theta^{\star}\|}{\sigma^2}
\end{align*}
Because $|\omega''| \leq 2 \omega'$ (see Lemma~\ref{lemma:omega-derivatives-inequality}) and $\omega' \geq 0$, we get
\begin{align*}
|\E\Delta\omega_{\theta}(X)| &\leq \left[ |\mu_1 - \mu_0| + |\sigma_1^2 - \sigma_0^2| \right] \int_0^1 \E \omega'\left( \sigma_{\lambda} Z + \mu_{\lambda} \right) d\lambda\\
 &\leq \dfrac{\|\theta - \ts\|(\| \theta \| + 2\|\ts\|)}{\sigma^2}  \int_0^1 \E \omega'\left( \sigma_{\lambda} Z + \mu_{\lambda} \right) d\lambda
\end{align*}

Lemma~\ref{lemma:normal-difference-integral} applied to the second term of (\ref{eq:expectation-x-delta}) works the same way, except with $\omega''$ and $\omega'''$ in place of $\omega'$ and $\omega''$. Use $|\omega''| \leq 2 \omega'$ again, along with $|\omega'''| \leq 4 \omega'$ (also from Lemma~\ref{lemma:omega-derivatives-inequality}) to find that
\begin{align*}
|\E\Delta\omega_{\theta}'(X)| &\leq \dfrac{2 \|\theta - \ts\|(\| \theta \| + 2\|\ts\|)}{\sigma^2} \int_0^1 \E \omega'\left( \sigma_{\lambda} Z + \mu_{\lambda} \right) d\lambda
\end{align*}

Lemma~\ref{lemma:omega-prime-expectation-bound} can be applied to this integral if we can verify the condition $|\mu_\lambda| \leq 2 \sigma_\lambda^2$ for all $0 \leq \lambda \leq 1$. Indeed, we've assumed $\langle \theta, \ts \rangle \geq \| \ts \|^2/2$ which implies (using Cauchy-Schwarz) $\| \ts \| \leq 2 \| \theta \|$, so
\begin{align*}
0 \leq \mu_\lambda &:= (1-\lambda) \tfrac{\|\ts \|^2}{\sigma^2} + \lambda \tfrac{\langle \theta, \ts \rangle}{\sigma^2}\\
 &\leq (1-\lambda) \tfrac{\|\ts \|^2}{\sigma^2} + \lambda \tfrac{2 \| \theta \|^2}{\sigma^2}\\
 &\leq 2 \sigma_\lambda^2
\end{align*}
By Lemma~\ref{lemma:omega-prime-expectation-bound},
\begin{align*}
\int_0^1 \E \omega'\left( \sigma_{\lambda} Z + \mu_{\lambda} \right) d\lambda &\leq \int_0^1 2e^{-(\mu_\lambda/\sigma_\lambda)^2/2} d\lambda\\
 &= \E_{\lambda \sim U[0, 1]} 2e^{-(\mu_\lambda/\sigma_\lambda)^2/2}\\
 &\leq \sup_{\lambda \in [0, 1]} 2e^{-(\mu_\lambda/\sigma_\lambda)^2/2}\\
 &\leq 2e^{-(s/r)^2/8}
\end{align*}
The last step comes from substituting the following lower bound for $\mu_\lambda/\sigma_\lambda$, derived using $\langle \theta, \ts \rangle \geq \| \ts \|^2 /2$ and $\| \theta \| \leq r \| \ts \|$.
\begin{align*}
\frac{\mu_\lambda}{\sigma_\lambda} &= \frac{(1-\lambda)s^2 + \lambda \langle \theta, \ts \rangle/\sigma^2}{\sqrt{(1-\lambda)s^2 + \lambda \| \theta \|^2/\sigma^2}}\\
 &\geq \frac{(1-\lambda)s^2 + \lambda s^2/2}{\sqrt{(1-\lambda)s^2 + r s^2}}\\
 &\geq \frac{s^2/2}{\sqrt{r^2 s^2}}\\
 &= s/(2r).
\end{align*}

We can also invoke Lemma~\ref{lemma:omega-prime-expectation-bound} to bound the expectation in the third term of (\ref{eq:expectation-x-delta}).
\begin{align*}
\E \omega'(\tfrac{\langle \ts, X \rangle}{\sigma^2}) &= \E \omega'(sZ + s^2)\\
 &\leq 2e^{-s^2/2}.
\end{align*}

Finally, returning to (\ref{eq:expectation-x-delta}), we can use the triangle inequality to bound the norm
\begin{align*}
\| \E [X \Delta \omega_{\theta}(X)] \| &\leq \| \theta - \ts \| \left( [\| \ts \| + 2 \| \theta \|] \frac{\| \theta \| + 2 \| \ts \|}{\sigma^2} 2e^{-(s/r)^2/8} + 2e^{-s^2/2} \right)\\
 &\leq \| \theta - \ts \| \left( 2[1+2r][r+2] s^2 e^{-(s/r)^2/8} + 2e^{-s^2/2} \right)\\
 &\leq \| \theta - \ts \| \left( 18r^2 s^2 e^{-(s/r)^2/8} + 2e^{-s^2/2} \right)\\
 &\leq \| \theta - \ts \| (36 r^4e^{-(s/r)^2/16}+2e^{-s^2/2})\\
 &\leq \| \theta - \ts \| \underbrace{38 r^4 e^{-(s/r)^2/16}}_{\gamma(s, r)/2}.
\end{align*}
(Recall that $\| M(\theta) - M(\ts) \|$ is twice as large as $\| \E [X \Delta \omega_{\theta}(X)] \|$.) The second-to-last step follows from the inequality $x e^{-x} \leq e^{-x/2}$; the last step follows from $r \geq 1$.\\

If $ s > 4r\sqrt{\log(76r^4)} \asymp r\sqrt{\log r} $, we see that $\gamma(s, r)$ is less than one.
\end{proof}

In their equation (29), \cite{Balakrishnan2014} define a ``first order stability'' condition of the form
\begin{align*}
\| \nabla Q(M(\theta) | \theta) - \nabla Q(M(\theta) | \ts) \| \leq \lambda \| \theta - \ts \|
\end{align*}
They point out in their Theorem 1 that if this stability condition holds and if $Q(\cdot | \ts)$ is $\lambda$-strongly concave over a Euclidean ball, then $M$ is contractive on that ball.

As they state, the $Q(\cdot | \ts)$ for this problem is $1$-strongly concave everywhere; in fact, the defining condition holds with equality. Checking for first order stability with $\lambda=1$ by substituting the gradient derived in (\ref{eq:q-gradient}) we find
\begin{align*}
\| \nabla Q(M(\theta) | \theta) - \nabla Q(M(\theta) | \ts) \| &= \| 2 \E Y \omega(\tfrac{\langle \theta, Y \rangle}{\sigma^2}) - 2 \E Y \omega(\tfrac{\langle \ts, Y \rangle}{\sigma^2}) \|\\
 &= \| M(\theta) - M(\ts) \|
\end{align*}
Because $M(\ts) = \ts$ in our case, Theorem~\ref{contractive} is equivalent to first order stability $\D_{1/2, r}$ when $(s, r)$ are such that $\gamma < 1$.\\

Theorem 1 from \cite{Balakrishnan2014} still holds with the Euclidean ball replaced by any set with the necessary stability and strong concavity, in our case $ \D_{1/2,r} $. Thus the framework can be applied, but Theorem~\ref{contractive} also get us directly to the destination.

Another difference is that we need to take additional steps to show that the iterations stay in the region $\D_{1/2,r}$, whereas in the Euclidean ball that was automatic. Our proof of stability was accomplished by Lemmas~\ref{inner_prod_bound} and~\ref{norm_bound}. In general, this suggests an alternative strategy for establishing contractivity, at least when $M$ has a closed form: identify regions for which $\| M(\theta) - M(\ts) \|$ can be controlled.

\subsection{Concentration of $\hat{T}$}

\begin{proof}[Proof of Proposition~\ref{concentration}]
Our strategy is to bound the moment generating function. We will show that for $ 2\sigma^2t(1+2t\|\theta^{\star}\|^2) < 1 $,
\begin{equation*}
\mathbb{E}e^{t(\|Y\|^2-d\sigma^2-\|\theta^{\star}\|^2)} \leq e^{-t d\sigma^2}(1-2\sigma^2t(1+2t\|\theta^{\star}\|^2))^{-d/2}.
\end{equation*}

Write $ Y = \sigma Z_d + \eta \theta^{\star} $, where $ \eta $ is an independent symmetric Rademacher variable and $ Z_d $ follows a $ N(0, I_d) $ distribution. Then $ \|Y\|^2 = \sigma^2\|Z_d\|^2 + 2\sigma\eta\langle Z, \theta^{\star}\rangle + \|\theta^{\star}\|^2 $. Using the inequality $ e^{x}+e^{-x} \leq 2e^{x^2/2} $, note that $ \mathbb{E}e^{2t\sigma\eta\langle Z_d, \theta^{\star}\rangle} \leq e^{2t^2\sigma^2|\langle Z_d, \theta^{\star}\rangle|^2} \leq e^{2t^2\sigma^2\|Z_d\|^2\|\theta^{\star}\|^2} $. Thus, we have shown that
\begin{equation*}
\mathbb{E}e^{t(\|Y\|^2-\|\theta^{\star}\|^2)} \leq \mathbb{E}e^{\|Z_d\|^2\sigma^2t(1+2t\|\theta^{\star}\|^2)}.
\end{equation*}
Since $ \|Z\|^2 $ follows a $ \chi^2_d $ distribution, we can use the chi-square moment generating function to write 
\begin{equation*}
\mathbb{E}e^{\|Z_d\|^2\sigma^2t(1+2t\|\theta^{\star}\|^2)} = (1-2\sigma^2t(1+2t\|\theta^{\star}\|^2))^{-d/2},
\end{equation*}
$ 2\sigma^2t(1+2t\|\theta^{\star}\|^2) < 1 $.

Using the inequality $ -\log(1-x) \leq x + 2x^2 $ for $ |x| \leq 1/2 $, we also have
\begin{equation*}
\mathbb{E}e^{t(\|Y\|^2-d\sigma^2-\|\theta^{\star}\|^2)} \leq e^{2t^2d\sigma^2\|\theta^{\star}\|^2+4\sigma^4t^2d(1+2t\|\theta^{\star}\|^2)^2},
\end{equation*}
for $ 2\sigma^2t(1+2t\|\theta^{\star}\|^2) < 1/2 $. Since $ s \geq 1 $ and $ t < 1/(8\sigma\|\theta^{\star}\|) $ also satisfy this restriction on $ t $, we have
\begin{equation*}
\mathbb{E}e^{t(\|Y\|^2-d\sigma^2-\|\theta^{\star}\|^2)} \leq e^{2d\sigma^2t^2(\|\theta^{\star}\|^2+2\sigma^2(1+s/4)^2)} \leq e^{9 d\|\theta^{\star}\|^2\sigma^2t^2}.
\end{equation*}
By the standard Chernoff method for bounded the tail of iid sums, we have
\begin{equation*}
\mathbb{P}(|\hat{T}-\|\theta^{\star}\|^2| > \epsilon) \leq 2\inf_{t<n/(8\sigma\|\theta^{\star}\|)}e^{-t\epsilon+9d\|\theta^{\star}\|^2\sigma^2t^2}.
\end{equation*}
The optimal choice of $ t $ is $ n\epsilon/(18d\sigma^2\|\theta^{\star}\|^2) $, producing a final bound of 
\begin{equation*} 
2\exp\{-n\epsilon^2/(36d\sigma^2\|\theta^{\star}\|^2)\},
\end{equation*} 
provided $ \epsilon < 5d\sigma\|\theta^{\star}\| $.
\end{proof}

\section*{Acknowledgements}

The authors would like to thank Sivaraman Balakrishnan and Andrew R. Barron for useful discussions that occurred at Yale in January 2015.

\bibliographystyle{imsart-nameyear}
\bibliography{EMReferences}

\end{document}